\documentclass[letterpaper, 10 pt, conference]{ieeeconf}  

\IEEEoverridecommandlockouts      

\usepackage[bmargin=25.4mm, rmargin=19.1mm, lmargin=19.1mm, tmargin=19.1mm]{geometry}
\usepackage{graphicx}
\usepackage{xcolor}
\usepackage{dsfont}
\usepackage{subcaption}
\usepackage{mwe}
\usepackage{url}

\usepackage{amsmath} 
\usepackage{amssymb}  
\usepackage{tikz}
\usetikzlibrary{calc,trees,positioning,arrows,chains,shapes.geometric,%
    decorations.pathreplacing,decorations.pathmorphing,shapes,%
    matrix,shapes.symbols}
\usetikzlibrary{quotes,angles}    
\tikzset{
>=stealth',
  punktchain/.style={
    rectangle, 
    rounded corners, 
    draw=black, very thick,
    text width=10em, 
    minimum height=3em, 
    text centered, 
    on chain},
  line/.style={draw, thick, <-},
  element/.style={
    tape,
    top color=white,
    bottom color=blue!50!black!60!,
    minimum width=8em,
    draw=blue!40!black!90, very thick,
    text width=10em, 
    minimum height=3.5em, 
    text centered, 
    on chain},
  every join/.style={->, thick,shorten >=1pt},
  decoration={brace},
  tuborg/.style={decorate},
  tubnode/.style={midway, right=2pt},
}

\usepackage{amsthm}
\newtheorem{theorem}{Theorem}
\newtheorem{definition}{Definition}
\newtheorem{remark}{Remark}

\newcommand{\bal}[1] {\ensuremath{\left(\begin{array}{#1}}}
\newcommand{\ear} {\ensuremath{\end{array}\right)}}

\newcommand{\bals}[1] {\ensuremath{\left[\begin{array}{#1}}} 
\newcommand{\ears} {\ensuremath{\end{array} \right] }} 

\let\leq\leqslant
\let\geq\geqslant

\newcommand{\calI}{\ensuremath{\mathcal{I}}}

\newcommand{\bmat}{\begin{matrix}}
\newcommand{\emat}{\end{matrix}}
\newcommand{\bbm}{\begin{bmatrix}}
\newcommand{\ebm}{\end{bmatrix}}
\newcommand{\bpm}{\begin{pmatrix}}
\newcommand{\epm}{\end{pmatrix}}
\newcommand{\bse}{\begin{subequations}}
\newcommand{\ese}{\end{subequations}}
\newcommand{\beq}{\begin{equation}}
\newcommand{\eeq}{\end{equation}}
\newcommand{\ben}{\begin{enumerate}}
\newcommand{\een}{\end{enumerate}}
\newcommand{\beni}{\renewcommand{\labelenumi}{\roman{enumi}.}
\renewcommand{\theenumi}{\roman{enumi}}\begin{enumerate}}
\newcommand{\eeni}{\end{enumerate}\renewcommand{\labelenumi}{\arabic{enumi}.}
\renewcommand{\theenumi}{\arabic{enumi}}}
\newcommand{\bena}{\renewcommand{\labelenumi}{\alpha{enumi}.}
\renewcommand{\theenumi}{\alpha{enumi}}\begin{enumerate}}
\newcommand{\eena}{\end{enumerate}\renewcommand{\labelenumi}{\arabic{enumi}.}
\renewcommand{\theenumi}{\arabic{enumi}}}
\newcommand{\bit}{\begin{itemize}}
\newcommand{\eit}{\end{itemize}}

\title{\LARGE \bf Cooperative decentralised circumnavigation with application to algal bloom tracking}

\author{Joana Fonseca, Jieqiang Wei, Karl H. Johansson and Tor Arne Johansen
\thanks{
This work is supported by Knut and Alice Wallenberg Foundation, Swedish Research Council, Swedish Foundation for Strategic Research, Research Council of Norway, CoE AMOS grant number 223254, and MASSIVE project grant number 270959.}
\thanks{Joana Fonseca, Jieqiang Wei, Karl H. Johansson are with the ACCESS Linnaeus Centre, School of Electrical Engineering and Computer Science. 
 KTH Royal Institute of Technology,
 SE-100 44 Stockholm, Sweden.
 {\tt\small \{jfgf, jieqiang, kallej\}@kth.se}.
}
\thanks{Tor Arne Johansen is with Department of Engineering Cybernetics,
Centre for Autonomous Marine Operations
and Systems (NTNU AMOS),
Norwegian University of Science
and Technology,
Trondheim N-7491, Norway.
{\tt\small tor.arne.johansen@ntnu.no}.
}
}

\begin{document}

\maketitle
\thispagestyle{empty}
\pagestyle{empty}

\begin{abstract}

Harmful algal blooms occur frequently and deteriorate water quality. A reliable method is proposed in this paper to track algal blooms using a set of autonomous surface robots. 
A satellite image indicates the existence and initial location of the algal bloom for the deployment of the robot system. The algal bloom area is approximated by a circle with time varying location and size. This circle is estimated and circumnavigated by the robots which are able to locally sense its boundary.
A multi-agent control algorithm is proposed for the continuous monitoring of the dynamic evolution of the algal bloom.
Such algorithm comprises of a decentralised least squares estimation of the target and a controller for circumnavigation.
We prove the convergence of the robots to the circle and in equally spaced positions around it. 
Simulation results with data provided by the SINMOD ocean model are used to illustrate the theoretical results.

\end{abstract}

\section{INTRODUCTION}

All over the world, the phenomena of harmful algal blooms occurs frequently. Plenty of research has been done regarding the nature of this phenomena, its causes and its impact. Note that, for instance, according to \cite{Harmful}, this phenomena is worth our best efforts to track as "Harmful algal blooms (HABs) cause human illness, large-scale mortality of fish, shellfish, mammals, and birds, and deteriorating water quality". Throughout this paper we'll be using simulated data of these algal blooms in the Norwegian sea. 

As a motivating example for this research, we can see how the use of autonomous surface vehicles (ASVs) allows to perform a series of measurement runs over a long period of time at sea \cite{asv}.  Hence, we believe a good solution relies on a system of ASVs with measuring abilities paired up with a satellite.

One may wonder why the satellite data is not enough for the problem of tracking these algal blooms. This could be a solution if we are interested in perhaps obtaining images and further studying them. But this is not the case. Our goal is to persistently track the different fronts of the algal bloom with surface agents close enough to the field to provide valuable data. With a satellite, there is a very low frequency of measurements. For instance, the data with which we simulate on this paper consists of two low quality images per day of the algal bloom. This satellite is not geosynchronous so it can only measure a specific area of the earth periodically. Also, the quality is commonly low due to clouds or other atmosphere obstacles. 
Therefore, a good solution relies on surface agents as well. We represent that idea in Fig.\ref{earth}, where the paper planes are the robots.

\begin{figure}
    \centering
    \includegraphics[width=7cm,trim={0 0.2cm 0 1cm},clip]{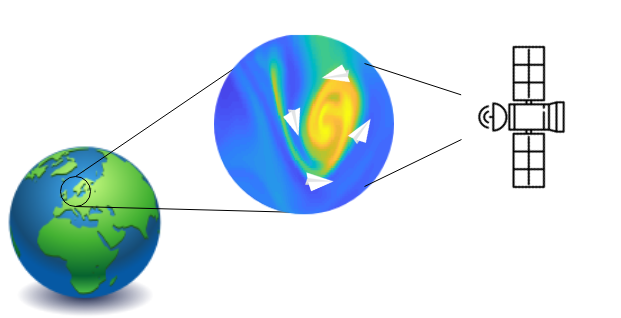}
    \caption{Tracking an algal bloom using a multi-robot system with local sensors and GPS}
    \label{earth}
\end{figure}

However, we seem to still be on the brink of discovering the best methods to consistently and efficiently track and circumnavigate these algal blooms.
In \cite{egerstedt2001formation} a path following algorithm is proposed for formation control of a multi-agent system. The authors prove that if the tracking errors are bounded, their method stabilises the formation error. However, it is assumed that there is perfect information on the path to follow. For our problem, we would like to estimate the target, design the path and control the multi-agent system.
In \cite{dimarogonas2008stability} and \cite{cao2007controlling} a control law for distance-based formation control which guarantees stability is proposed. Also in section 6.3.1 of \cite{Sun2018}, where target tracking is considered, they use distance-based formation control. However, a distance-based protocol does not suit our target tracking problem.
In \cite{Franchi} a protocol for target tracking in 3D is designed with guaranteed collision avoidance. However, it is assumed that the target is a fixed object that may move and rotate but never change its shape, as in our case. 
In \cite{swarmAni} and \cite{swarmLi}, controllers are synthesised for a swarm of robots to generate a desired two-dimensional geometric pattern specified by a simple closed planar curve. It is assumed that the shape is given to the swarm and not estimated in real-time. This is not true for our case.
In \cite{Shames2012} an adaptive protocol to circumnavigate around a moving point is proposed, e.g., the fish tracking problem. They used adaptive estimation for point tracking with known constant distance and they use just one agent.
In \cite{distance} the problem and assumptions are similar as the previous paper but here they apply sliding mode control. Even though we also assume the agents can measure the distance to the target, these papers assume that the target to track is a  point.
In \cite{iros2} and\cite{iros} the agent has access to the bearing measure towards the target. This assumption differs from ours as we assume we measure only the distance.
Some closely related results \cite{Johanna}, \cite{bearing} and \cite{bearing2} use either bearing or distance measurements to the target while using a network of autonomous agents to circumnavigate. While relevant, these results do not apply to a shape but only to a moving point with circumnavigation within a preset distance.
In \cite{iros3} they devise an algorithm such that one robot can circumnavigate a circular target from a prescribed radius using the bearing measurement. Even though they circumnavigate a circle, they do so at a prescribed distance and it is assumed that the robot is capable of measuring the bearing to the target, which is not the case in our paper.

The main contribution of this paper is a distributed algorithm that includes the real time estimation of the target and devises a control protocol to apply to each agent. We focus both on mathematical guarantees of bounded convergence and on physical restrictions for implementation. The present algorithm was tested using data from SINMOD of an algal bloom target in the Norwegian sea. 

\subsection{Notations}\label{s: preli}

The notations used in this paper are fairly standard. $\mathds{1}$ is an array of ones. $\|\cdot\|_p$ denotes the $\ell_p$-norm and  the $\ell_2$-norm is denoted simply as $\|\cdot\|$ without a subscript.
We define a rotation matrix $E$ as  
\begin{equation}
E =
\begin{bmatrix}
0 & 1 \\
-1 & 0
\end{bmatrix} .
\end{equation}

\subsection{Problem Definition}\label{s: definition}

Having a circular moving and varying algal bloom shape we wish to circumnavigate it using a system of robots. Each robot is equipped with a sensor that indicates the distance to the boundary, including whether it is inside the shape or outside. First step is the estimation of the parameters of the algal bloom circle, that is, its centre and radius for every time instance. Second step is to design a control law for all robots to circumnavigate the shape according to the estimated circle. 
Furthermore, it should be proved that the estimated parameters converge to the real ones and the robots converge to the boundary, while circumnavigating it. Plus, they should be equally distributed along such boundary. 

\subsection{Outline}\label{s: outlay}

The remaining sections of this paper are organised as follows. In Section \ref{s: problem}, the main problem of interest is formulated. 
The main results are presented in Section \ref{s: main}, where the protocol is designed and its proofs of convergence presented.
Some simulations presenting the performance of the proposed algorithm are given in Section \ref{s:simulation}. 
Concluding remarks and future directions come in Section \ref{s: conclusion}.

\section{PROBLEM STATEMENT}\label{s: problem}

In this paper we consider the problem of tracking a circular shape using a multi-robot system and a satellite. This shape may be very irregular and unstable over time. We assume the shape is close to a circle. An initial image of the algal bloom confirms such assumption, as seen in Fig.\ref{encircle}, and then we can decide to use our algorithm and deploy the agents. 

\begin{figure}
  \begin{subfigure}[t]{.24\textwidth}
    \centering
    \includegraphics[width=\linewidth,trim={0.4cm 0 0.8cm 0.6cm},clip]{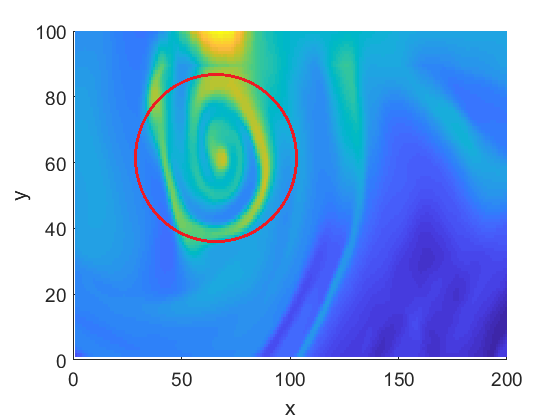}
  \end{subfigure}
  \begin{subfigure}[t]{.24\textwidth}
    \centering
    \includegraphics[width=\linewidth,trim={0.4cm 0 0.8cm 0.6cm},clip]{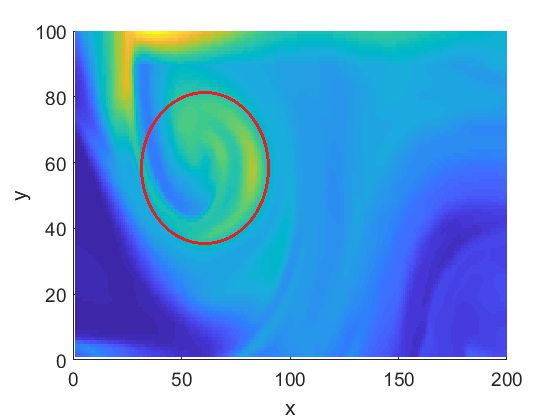}
  \end{subfigure}
  \begin{subfigure}[t]{.24\textwidth}
    \centering
    \includegraphics[width=\linewidth,trim={0.4cm 0 0.8cm 0.6cm},clip]{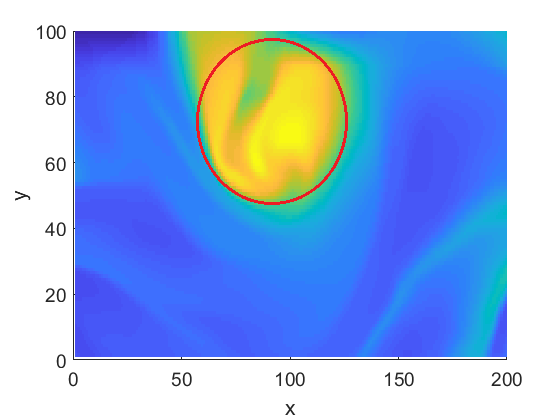}
  \end{subfigure}
  \begin{subfigure}[t]{.24\textwidth}
    \centering
    \includegraphics[width=\linewidth,trim={0.4cm 0 0.8cm 0.6cm},clip]{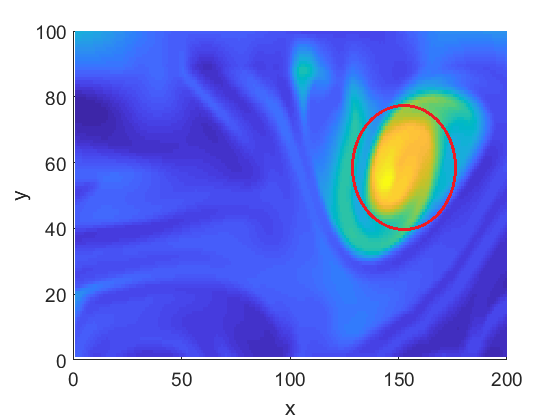}
  \end{subfigure}
  \caption{Time-lapse of the algal bloom progression. There is approximately half a day between each image. Warm colours (yellow, orange, green) indicate high density of algal and cold colours (blues) indicate low density of algal.}
  \label{encircle}
\end{figure}

We define this circle as
\begin{equation}\label{t_est}
\begin{aligned} 
(\mathbf{c}(t),r(t)) \in \mathbb{R}^3,
\end{aligned}
\end{equation}
where $\mathbf{c}(t)= (x(t), y(t))$ and $r(t)$ are the centre and the radius of the circle, respectively.
After confirming the algal bloom is close enough to a circle we can then estimate it by our robot's measurements. This estimate is represented as $(\mathbf{\hat c}(t),\hat r(t)) \in \mathbb{R}^3$. Note that the usage of this circle does not compromise the generality of the algorithm. Instead, it guarantees a smooth circumnavigation for any irregular shape close to a circle. Similar algorithm can be done for shapes that can be approximated by ellipsoids, but we present a simpler case, namely with circle shapes, due to lack of space. 

In order to solve this tracking problem we use two types of tools: a satellite and a system of robots. The satellite obtains data from the target in the form an image depending on the weather. Then, it calculates by image processing the possible initial centre and radius of such circle and shares it with the robots so they can move towards the target and initiate circumnavigation. 
So, the satellite would provide initial estimates $\mathbf{\hat{c}}(0)= (\hat{x}(0), \hat{y}(0))$ and $\hat{r}(0)$.
The robots constantly measure their distances to the target's boundary, as well as whether they're inside or outside the target, and share it with the other robots. Each robot has access to its GPS position and to the position of the robot in front of it. This communication scheme is represented in Fig.\ref{comuni}. Note that $\mathbf{\hat{c}}_0$, $\hat{r}_0$ represent the initial values for the estimate of the target. Values such as $B_i$, $p_i$, $D^b_i$, $S_i$ will be soon properly defined.

\begin{figure}
    \centering
    \includegraphics[width=8.5cm,trim={0 0.2cm 0 0.1cm},clip]{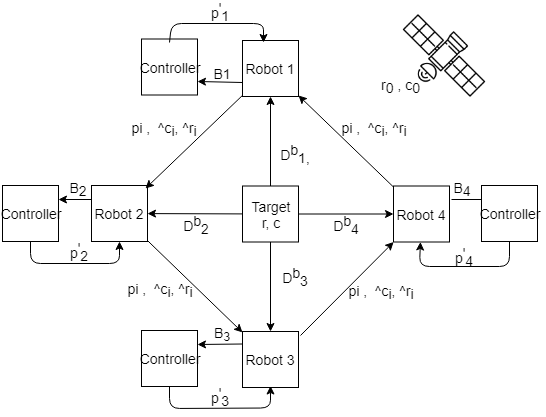}
    \caption{System setup and communication architecture}
    \label{comuni}
\end{figure}

The system of robots will jointly circumnavigate the target and provide real time information of different fronts. We define we have $n$ agents and, using the satellite information, they are initialised at positions $\mathbf{p_i}(0), i\in\calI$, which are outside of the shape and form a counterclockwise directed ring on the surface.
The kinematic of the agents is of the form
\begin{align} \label{e:dyn_agent}
\mathbf{\dot p_i} = \mathbf{u_i}, \qquad   i\in\calI,
\end{align}
where $\mathbf{p_i}$ is a vector that contains the position $p_i = [ x_i,  y_i ]^{\top}\in\mathbb{R}^2$ and $\mathbf{u_i}\in\mathbb{R}^2$ is the control input.

In order to avoid the agents concentrating in some region, in which case they may loose information on other fronts, we would like to space the agents equally along the defined circle.
Therefore, we define that the counterclockwise angle between the vector $\mathbf{p_i}-\mathbf{\hat{c}}$ and $\mathbf{p_{i*1}} -\mathbf{\hat{c}}$ is denoted as $\beta_i$ for $i=1,\dots,n-1$, and the angle between $\mathbf{p_n}-\mathbf{\hat{c}}$ and $\mathbf{p_1}-\mathbf{\hat{c}}$ is denoted as $\beta_n$,
\medskip
 \begin{equation}\label{beta}
 \begin{aligned} 
  \beta_{i} = & \angle(\mathbf{p_{i+1}}-\mathbf{\hat{c}},\mathbf{p_i}-\mathbf{\hat{c}}), \qquad i=1,\dots,n-1 \\
  \beta_n = & \angle(\mathbf{p_1}-\mathbf{\hat{c}},\mathbf{p_n}-\mathbf{\hat{c}}).
 \end{aligned}
 \end{equation}

Notice that in this case, 
\begin{align}
    \beta_i(0)\geq 0, \quad \textnormal{and} \quad \sum_{i=1}^{n} \beta_i(0) = 2\pi.
 \end{align} 
This is represented in figure Fig.\ref{Scheme1}.

\begin{figure}
\centering
\begin{tikzpicture}
  [
    scale=1.7,
    >=stealth,
    point/.style = {draw, circle,  fill = black, inner sep = 1pt},
    dot/.style   = {draw, circle,  fill = black, inner sep = .2pt},
  ]
  
  \def\rad{0.3}
  \node (n1) at (1.5,0) [point, label = { }]{};
  \draw (n1) circle (\rad);

  \def\rad{0.2}
  \node (n2) at (0.6,0.8) [point, label = { }]{};
  \draw (n2) circle (\rad);
  
  \def\rad{0.3}
  \node (n3) at (-0.9,1.2) [point, label = { }]{};
  \draw (n3) circle (\rad);
  
  \def\rad{0.3}
  \node (n4) at (-0.9,0) [point, label = { }]{};
  \draw (n4) circle (\rad);
  
  \def\rad{1.2}
  \node (est) at (0,0) [point, label = {below right:$\hat{c}$}]{};
  \draw (est) circle (\rad);

  \node (c) at (0,0) [point, label] {};
  \node (n1) at (1.5,0) [ ] {};
  \node (n2) at (0.6,0.8) [ ] {};
  \node (n3) at (-0.9,1.2) [] {}; 
  \node (n4) at (-0.9,0) [ ] {};

    \draw[dotted]
    -- (0,0) coordinate (centre) node[left] {}
    -- (1.5,0) coordinate (p1) node[below right] {$p_1$};
    \draw[dotted]
    -- (0,0) coordinate (centre) node[left] {}
    -- (0.7,1) coordinate (p2) node[above right] {$p_2$};
    \draw[dotted]
    -- (0,0) coordinate (centre) node[left] {}
    -- (-0.9,1.2) coordinate (p3) node[above right] {$p_3$};
    \draw[dotted]
    -- (0,0) coordinate (centre) node[left] {}
    -- (-0.9,0) coordinate (p4) node[below right] {$p_4$};

    \draw
    pic["$\beta_1$", draw=orange, ->, angle eccentricity=1.4, angle radius=0.6cm] {angle=n1--c--n2};
    \draw
    pic["$\beta_2$", draw=orange, ->, angle eccentricity=1.4, angle radius=0.6cm] {angle=n2--c--n3};
    \draw
    pic["$\beta_3$", draw=orange, ->, angle eccentricity=1.4, angle radius=0.6cm] {angle=n3--c--n4};
    \draw
    pic["$\beta_4$", draw=orange, ->, angle eccentricity=1.4, angle radius=0.6cm] {angle=n4--c--n1};
    
\end{tikzpicture}   
\caption{Example scheme of the system with four agents at positions $p_1$, $p_4$, $p_3$, $p_2$. Note how each of them has access to the distance to the boundary, which represented by a circumference.} \label{Scheme1}
\end{figure}
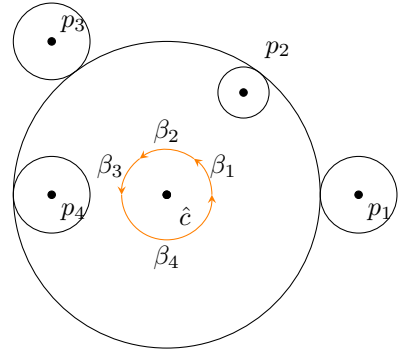

We can define the distance of each agent $i$ to the centre as $D^c_i(t)=\|\mathbf{\hat{c}}-\mathbf{p}_i(t)\|$. Since we don't have access to the centre $c$, the distance to the estimated centre is represented as $\hat{D}^c_i(t)=\|\mathbf{\hat{c}}(t)-\mathbf{p}_i(t)\|$. Then, knowing that each robot has access to its distance to the boundary, we can define it as 
\begin{equation}\label{distances}
\begin{aligned} 
D^b_i(t) = D^c_i(t)-r(t).
\end{aligned} 
\end{equation}
This value is constantly measured by each agent, as in  Fig. \ref{Scheme1} and Fig. \ref{Scheme}. Note that $D^b_i(t)$ is positive if the agent is outside the algal bloom area or negative if it is inside the algal bloom area. For example, if an agent $i$ is inside of the circle about 5 meters then $D^b_i=-5$ and if this agent is outside of the circle about 5 meters then $D^b_i=5$.

\begin{figure}
\centering
\begin{tikzpicture}
  [
    scale=2,
    >=stealth,
    point/.style = {draw, circle,  fill = black, inner sep = 1pt},
    dot/.style   = {draw, circle,  fill = black, inner sep = .2pt},
  ]
  
  \def\rad{0.4}
  \node (n3) at (1.2,0) [ ]{};
  \draw (n3) circle (\rad);
  
  \def\rad{0.9}
  \node (real) at (-0.3,0.2) [point, label = {above left:$c$}]{};
  \draw (real) circle (\rad); 
  \node (n5) at (-1.2,0.2) [point, label] {};
  \draw[-] (real) -- node (d) [label = {above left:$r$}] {} (n5);
  
  \def\rad{0.8}
  \node (est) at (0,0) [point, label = {below right:$\hat{c}$}]{};
  \draw (est) circle (\rad);

  \node (n1) at (0.7,1) [point, label] {};
  \node (n2) at +(-180:\rad) [point, label] {};
  \node (n3) at (1.2,0) [point, label] {};
  \node (n4) at (0.8,0) [point, label] {};
  
  \draw[-] (est) -- node (b) [label = {below:$\hat{r}$}] {} (n2);
  \draw[-] -- node (c) {} (n3);
  \draw[-] (n3) -- node (n3) [label = {below:$\hat D^b_i$}] {} (n4);
  \draw[-] (est) -- node (n1) [label = {left:$\hat D^c_{i+1}$}] {} (n1);
  
    \draw[dotted]
    (0.7,1) coordinate (a) node[right] {$p_{i+1}$}
    -- (0,0) coordinate (b) node[left] {}
    -- (1.2,0) coordinate (c) node[above right] {$p_{i}$};
    
    \draw
    pic["$\beta_i$", draw=orange, ->, angle eccentricity=1.4, angle radius=0.6cm] {angle=c--b--a};

    \draw[decoration={random steps, amplitude=2mm}, decorate] (-0.3,0.2) circle (0.9);
    
\end{tikzpicture}   
\caption{Scheme of the estimated $\hat{c}$, $\hat{r}$ and the real target $c$, $r$ as well as the angle $\beta_i$ between two agents at $p_{i+1}$ and $p_i$ } \label{Scheme}
\end{figure}
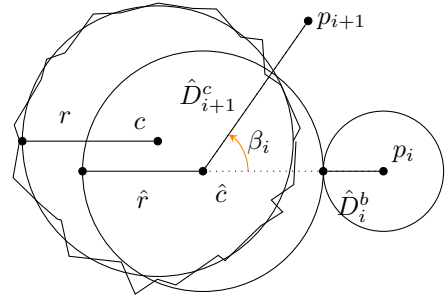

\begin{definition}[Circumnavigation]
When the target is stationary, i.e., $\mathbf{c}$ and $r$ are constant, circumnavigation is achieved if the agents   
\begin{enumerate}
\item move in a counterclockwise direction on the boundary of the target, and
\item are equally distributed along the circle, i.e., $\beta_i =\frac{2\pi}{n} $.
\end{enumerate}
More precisely, we say that the circumnavigation is achieved asymptotically if the previous aim is satisfied for $t\rightarrow\infty$. 

For the case with time-varying target, we assume that $\|\dot{\mathbf{c}}\|\leq\varepsilon_1$ and $|\dot{r}|\leq\varepsilon_2$ for some positive constant $\varepsilon_1$ and $\varepsilon_2$.
\end{definition}

Now we are ready to pose the problem of interest that will be solved in the following sections.

\medskip

\textbf{Problem 1.} 
Design estimators for $\mathbf{c}(t)$ and $r(t)$ when both the distance measures \eqref{distances} and GPS positions are available to each agent. Design the control input $\mathbf{u}_i$ for all the agents such that for some positive $\varepsilon_1$, $\varepsilon_2$,
\begin{align}
    &\|\mathbf{\dot c }\|\leq\varepsilon_1\\
    &|\dot r |\leq\varepsilon_2,
\end{align}
there exist positive $K_1$, $K_2$ and $K_3$ satisfying
\begin{align}
    &\lim\sup_{t\rightarrow\infty} \|\hat{\mathbf{c}}(t)-\mathbf{c}(t)\| \leq K_1 \varepsilon_1,\\
    &\lim\sup_{t\rightarrow\infty} |\hat{r}(t)-r(t)| \leq K_2 \varepsilon_2,\\
    &\lim\sup_{t\rightarrow\infty} |D_i^b| \leq K_3 \varepsilon_2,\\
    &\lim_{t\rightarrow\infty} \beta_i = \frac{2\pi}{n}.
\end{align}

\section{MAIN RESULTS}\label{s: main}

Here follows our solution for Problem 1. We consider $n$ agents at positions $p_i(t)$ and we assume all of them are capable of measuring their distances $D_i^b(t)$ to the target boundary including whether they're inside ($D_i^b(t)$ is negative) or outside ($D_i^b(t)$ is positive) of it. Then, they should estimate $(\mathbf{c}(t),r(t))$ from their shared measurements. For robustness, they update their estimates by taking the average of the estimated variables by the $n$ agents. Also, if one or more agents suffered faulty measurements due to bad conditions or failure, the system is ready to support that situation by using the remaining agent's estimates. Each agent calculates its desired velocity taking into account its angle $\beta_{i}(t)$ to the next agent and its distance to the boundary. The scheme on Fig.\ref{Scheme5} summarises this algorithm loop.

First step is the estimation of the circle. Having all the agents constantly measuring $D^b_i$ we can fit a unique circle as in Fig.\ref{Scheme1}, given that the target shape is a circle. Mathematically, such circle can be obtained through triangulation and, therefore, we would only need 3 agents to obtain a unique solution. However, for better coverage of all the fronts and for robustness, more than 3 agents are considered. Note that, in this paper's result section we used 4 agents. So, we apply the least squares method to obtain the approximated circle as in \eqref{lsq}.
\begin{align} \label{lsq}
 &\min_{\mathbf{\hat{c}},\hat{r}}\sum_i^n \left ( \|\mathbf{p}_i-\mathbf{\hat{c}}\| - (\hat{r}+D^b_i) \right )^2. \\ \nonumber
 & s.t \qquad \hat r>0. 
\end{align}


Now, we want to obtain the desired control input $\mathbf{u_i}(t)$ using the previously measured and estimated variables. The total velocity of each agent comprises of two sub-tasks: approaching the target and circumnavigating it. Therefore we define the direction of each agents towards the centre of the target as the bearing $\psi_i(t)$,
\begin{equation} \label{bearing}
 \psi_i(t) = \frac{\mathbf{\hat{c}}(t) - p_i(t)}{\hat{D^c_i}(t)} = \frac{\mathbf{\hat{c}}(t) - p_i(t)}{\|\mathbf{\hat{c}}(t)-p_i(t)\| \\}. 
\end{equation}

Note that $\psi_i$ in \eqref{bearing} is not well-defined when $\hat{D^c_i}=0$, thus we will prove that this singularity is avoided for all time $t\geq 0$ in Theorem \ref{Theo1}.

In order to build the control, we need to define $\mathbf{\dot{\hat{c}}}(t)$ and $\dot{\hat{r}}(t)$. Even though $\mathbf{c}(t)$ and $r(t)$ are continuous functions, our estimates $\mathbf{\hat{c}}(t)$ and $\hat r(t)$ are, inevitably, a discrete function. Therefore, for each time interval $\Delta_T$, we define $\mathbf{\dot{\hat{c}}}(t)$ and $\dot{\hat[ r]}(t)$ as
\begin{align} \label{dotc}
 &\mathbf{\dot{\hat{c}}}(t) = \frac{\mathbf{\hat{c}}(t+\Delta_T)-\mathbf{\hat{c}}(t)}{\Delta_T} \\ \label{dotr}
 &\dot{\hat{r}}(t) = \frac{\hat r(t+\Delta_T)-\hat r(t)}{\Delta_T}
\end{align}

The first sub-task is related to the bearing $\psi_i(t)$ and the second one is related to its perpendicular, $E\psi_i(t)$.
Therefore, the control law for each agent $i$ is
\begin{equation}\label{e:control-adapt}
\mathbf{u_i} =  
\mathbf{\dot{\hat{c}}} + ((\hat{D^c_i} - \hat r) - \dot{\hat{r}})\psi_i + \beta_i \hat{D^c_i} E \psi_i
\end{equation}

\begin{remark}
    Note that for implementation we would define $U_i$ as the control input for each agent $i$. Then, $U_i$ must have a maximum absolute value $u_{max}$ since the maximum velocity of the agent would be limited as well. $U_i$ could either be represented as $U_i=\delta u_i$, being $\delta$ some positive parameter for tuning, or represented as the saturation function: if $\|u_i\| > u_{max}$ then $U_i=\frac{u_{max}}{\|u_i\|}u_i$, else $U_i = u_i$.
\end{remark}

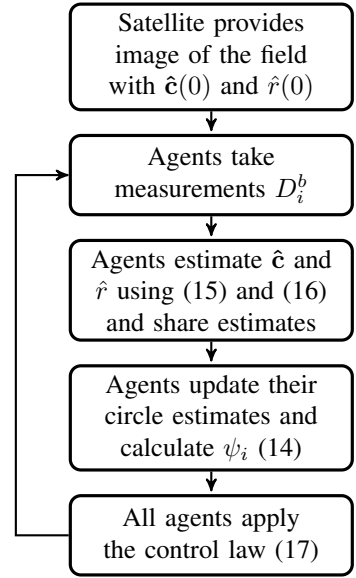
\begin{figure}
\centering
    \begin{tikzpicture}
  [node distance=.3cm, start chain=going below]
     \node[punktchain, join] (1) {Satellite provides image of the field with $\mathbf{\hat{c}}(0)$ and $\hat{r}(0)$};
     \node[punktchain, join] (2) {Agents take measurements $D_i^b$ };
     \node[punktchain, join] (3) {Agents estimate $\mathbf{\hat{c}}$ and $\hat{r}$ using \eqref{dotc} and \eqref{dotr} and share estimates};
     \node[punktchain, join] (4) {Agents update their circle estimates and calculate $\psi_i$ \eqref{bearing}};
     \node[punktchain, join] (5) {All agents apply the control law \eqref{e:control-adapt}};
     
     \draw[|-,-|,->, thick,] (5.west) -|+(-2em,0)|- (2.west);
  \end{tikzpicture}
\caption{Scheme of the algorithm run on the system} \label{Scheme5}
\end{figure}

\begin{theorem} \label{Theo1}
Consider the system \eqref{e:dyn_agent} with the control protocol \eqref{e:control-adapt}, and $\|\dot{\mathbf{c}}\|\leq\varepsilon_1$, $|\dot{r}|\leq\varepsilon_2$, then there exists $K_1$, $K_2$ and $K_3$ such that circumnavigation of the moving circle with equally spaced agents can be achieved asymptotically up to a bounded error, i.e.
\begin{align}\label{oi}
    &\lim\sup_{t\rightarrow\infty} \|\hat{\mathbf{c}}(t)-\mathbf{c}(t)\| \leq K_1 \varepsilon_1,\\\label{oi2}
    &\lim\sup_{t\rightarrow\infty} |\hat{r}(t)-r(t)| \leq K_2 \varepsilon_2,\\\label{oi3}
    &\lim\sup_{t\rightarrow\infty} |D_i^b| \leq K_3 \varepsilon_2,\\ \label{oi4}
    &\lim_{t\rightarrow\infty} \beta_i = \frac{2\pi}{n}.
\end{align}
\end{theorem}

\begin{proof}
The proof is divided into four parts. In the first part, we prove that \eqref{oi} and \eqref{oi2} hold. In the second part, we prove that the estimated distance $\hat{D^c_i}$ converges to the estimated radius $\hat{r}$, or in other words, that \eqref{oi3} holds. In the third part we prove that the singularity of the bearing $\psi_i(t)$ is avoided. In the last part, we show that the angle between the agents will converge to the average consensus for $n$ agents, $\beta_i = \frac{2\pi}{n}$, meaning \eqref{oi4} holds. We will assume the implementable controller is given by $U_i=\delta \mathbf{u}_i$.
\begin{enumerate}

    \item Firstly, we prove that \eqref{oi} and \eqref{oi2} hold. Having \eqref{lsq} we can see that its lowest minimum possible is zero. Then, all of its terms must be zero as well. So, we have that the following holds:
    \begin{align} \label{equal}
        \|p_i-\mathbf{\hat{c}}\| = \hat{r}+D^b_i  \qquad \forall i\in[1,...,n] 
    \end{align}
    Note that, geometrically, this corresponds to the Pythagoras theorem. The left part of the equality correspond to the sides of the triangle in $x$ and $y$ and the right side corresponds to the hypotenuse. Therefore, the only values for which the equality \eqref{equal} holds is for $\mathbf{\hat{c}}=\mathbf{c}$ and $\hat{r}=r$.
    However, $\mathbf{c}(t)$ and $r(t)$ are continuous functions and $\mathbf{\hat{c}}$ and $\hat r(t)$ are discrete functions. Therefore, instead of $\lim\sup_{t\rightarrow\infty} \|\hat{\mathbf{c}}(t)-\mathbf{c}(t)\| = 0$, we get $\lim\sup_{t\rightarrow\infty} \|\hat{\mathbf{c}}(t)-\mathbf{c}(t)\| \leq K_1 \varepsilon_1$. Being $K_1$ a parameter equal to the differentiation interval $\Delta_T$. Same applies for the radius $r(t)$.
    
    \item We prove that all agents reach the estimate of the boundary of the moving circles asymptotically, i.e., $\lim_{t\rightarrow\infty}{\hat{D^c_i}(t)} = \hat r(t)$, so \eqref{oi3} holds.     
        
    Consider the function $W_i(t):=\hat{D^c_i}(t)-\hat r(t)$ whose time derivative for $t\in[0,\tau_{\max})$ is given as
    \begin{align*}
        \dot W_i = & \frac{(\mathbf{\hat{c}}-p_i)^\top (\mathbf{\dot{\hat{c}}}-\dot{p}_i)}{\hat{D^c_i}} - \dot{\hat{r}} \\
         =  & -\frac{(\mathbf{\hat{c}}-p_i)^\top}{\hat{D^c_i}}\psi_i\delta (\hat{D^c_i} - \hat r - \dot{\hat{r}}) \\
         & \qquad -\frac{(\mathbf{c}-p_i)^\top}{\hat{D^c_i}}E\psi_i\delta \beta_i \hat{D^c_i} - \dot{\hat{r}} \\
         = & - \delta(\hat{D^c_i} - \hat{r} - \dot{\hat{r}}) - \dot{\hat{r}}
         = - \delta W_i.
    \end{align*}
    Hence for $t\in[0,+\infty)$, we have $\hat{D^c_i}(t) = \delta W_i(0) e^{-t} + \hat{r}(t)$ which implies $W_i$ is converging to zero exponentially.
    
    \item Now, we prove that $\psi_i$ in \eqref{bearing} is well-defined, or in  other words, that its singularity is avoided for all time $t\geq 0$, $\hat{D^c_i}\ne0$ $\forall t$.

    Having $\hat{D^c_i}(t) = \delta W_i(0) e^{-t} + \hat{r}(t)$ from the previous proof and knowing that $W_i(0)$ is always positive and that it converges to zero exponentially, we have that if $\hat{r}(t)>0$ then $\hat{D^c_i}(t)>0$, $\forall t$.
    
    So we would have to prove that $\hat{r}(t)>0$ $\forall t$. Given that we use the least squares method to obtain the estimate of the radius, we can see how one of the constraints guarantees that $\hat{r}(t)>0$ $\forall t$.
    
    Then we conclude that $\hat{D^c_i}\ne0$ $\forall t$ and that the bearing $\psi_i(t)$ is well defined $\forall t$.
    
    \item Finally, we show that the angle between the agents will converge to the average consensus for $n$ agents, $\beta_i = \frac{2\pi}{n}$, so \eqref{oi4} holds. 
    
    Firstly, note that we can write an angle between two vectors $\beta_i=\angle(v_2,v_1)$ as
    \begin{equation}
        \beta_i=2atan2((v_1\times v_2)\cdot z,\|v_1\|\|v_2\|+v_1\cdot v_2)
    \end{equation}
    and its derivative as
    \begin{equation}
        \dot\beta_i=\frac{\hat{v_1}\times z}{\|v_1\|}\dot{v_1}-\frac{\hat{v_2}\times z}{\|v_2\|}\dot{v_2}
    \end{equation}
    where $z = \frac{v_1\times v_2}{\|v_1\times v_2\|}, \hat{v_i} = \frac{v_1}{\|v_i\|}, i=1,2$.
    
    Then, for $v_1=p_i-\hat{c}$ and $v_2=p_{i+1}-\hat{c}$ we get
    \begin{equation*}
    \begin{aligned}
        \dot{\beta_i} &= \frac{\hat{v_1}\times z}{\|v_1\|}\dot{v_1}-\frac{\hat{v_2}\times z}{\|v_2\|}\dot{v_2} \\
        &= \frac{\hat{v_1}\times z}{\|v_1\|}\delta((\hat{D^c_i} - \hat r - \dot{\hat{r}})\psi_i + \beta_i \hat{D^c_i} E \psi_i) \\
        & \qquad -\frac{\hat{v_2}\times z}{\|v_2\|}\delta((\hat{D}^c_{i+1} - \hat r - \dot{\hat{r}})\psi_{i+1} \\
        & \qquad \qquad + \beta_{i+1} \hat{D^c_{i+1}}  E \psi_{i+1}) \\
        &= \delta (-\beta_{i} + \beta_{i+1}), \qquad i= 1,\ldots,n-1 \\ 
        \dot{\beta}_n &= \delta (-\beta_{n} + \beta_{1}).
    \end{aligned}
    \end{equation*}
    which can be written in a compact form as following
    \begin{align}
        \dot{\beta} = - \delta B^\top \beta \label{e:dyn-beta}
    \end{align}
    where $B$ is the incidence matrix of the directed ring graph from $v_1$ to $v_n$. 
    
    First, we note that the system \eqref{e:dyn-beta} is positive (see e.g., \cite{farina2000positive}), i.e., $\beta_i(t)\geq 0$ if $\beta_i(0)\geq 0$ for all $t\geq 0$ and $i\in\calI$. This proves the positions of the agents are not interchangeable.
    Second, noticing that $B^\top$ is the (in-degree) Laplacian of the directed ring graph which is strongly connected, then by Theorem 6 in \cite{Wei2018}, $\beta$ converges to consensus $\frac{2\pi}{n}\mathds{1}$.
    
    
\end{enumerate}
\end{proof}

\begin{remark}
    Note how the agent $A_i$ will necessarily maintain its relative position $p_i$ throughout the circumnavigation mission. In fact, we can prove that agent $A_i$ is always in position $p_i$.
\end{remark}
\begin{remark}
We proved both convergence of the angle to the average consensus for $n$ agents and convergence of these agents towards the boundary of the target up to a given bound. Therefore, we guarantee collision avoidance.
\end{remark}

\section{SIMULATION RESULTS}\label{s:simulation}

In this section, we present simulations for the protocol designed in section \ref{s: main}. We use the derived method for estimation of the target \eqref{lsq} and the controlling protocol for the agents \eqref{e:control-adapt}. For this section, we discretize the whole algorithm to be able to use it computationally. 

We use the target present in the images provided by SINMOD simulations \url{https://www.sintef.no/en/ocean/initiatives/sinmod/#/}. The present simulation corresponds to approximately 4 days of data and the target we obtained is approximately 1-3km in radius.

\begin{figure}
  \begin{subfigure}[t]{.49\textwidth}
    \centering
    \includegraphics[width=\linewidth,trim={0.5cm 0 0.8cm 0.6cm},clip]{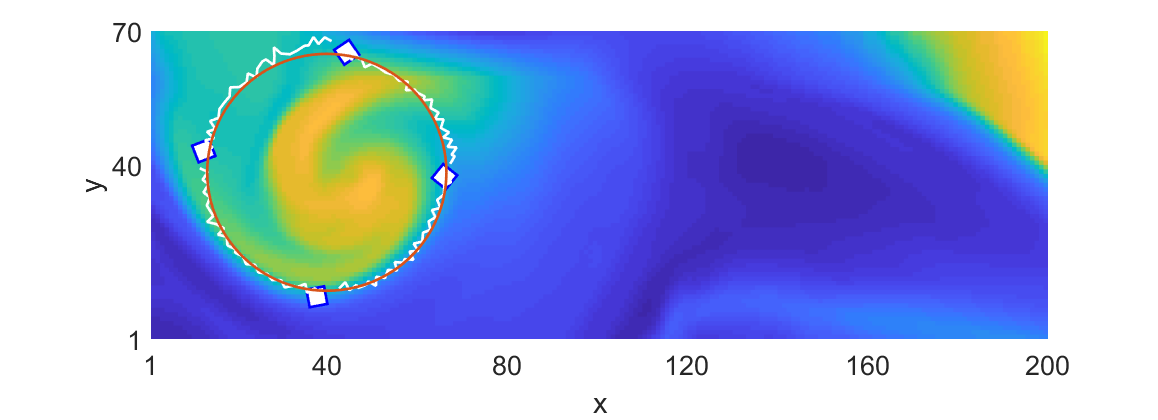}
  \end{subfigure}
  \begin{subfigure}[t]{.49\textwidth}
    \centering
    \includegraphics[width=\linewidth,trim={0.5cm 0 0.8cm 0.6cm},clip]{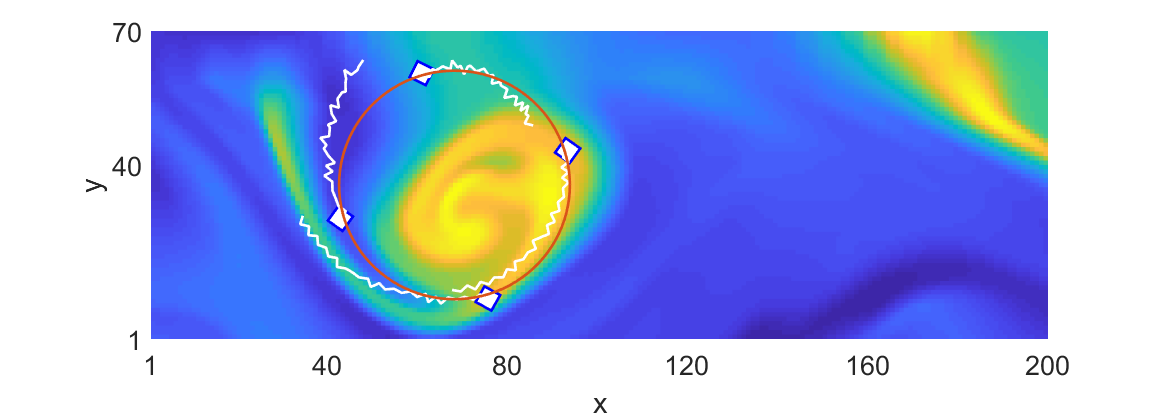}
  \end{subfigure}
  \begin{subfigure}[t]{.49\textwidth}
    \centering
    \includegraphics[width=\linewidth,trim={0.5cm 0 0.8cm 0.6cm},clip]{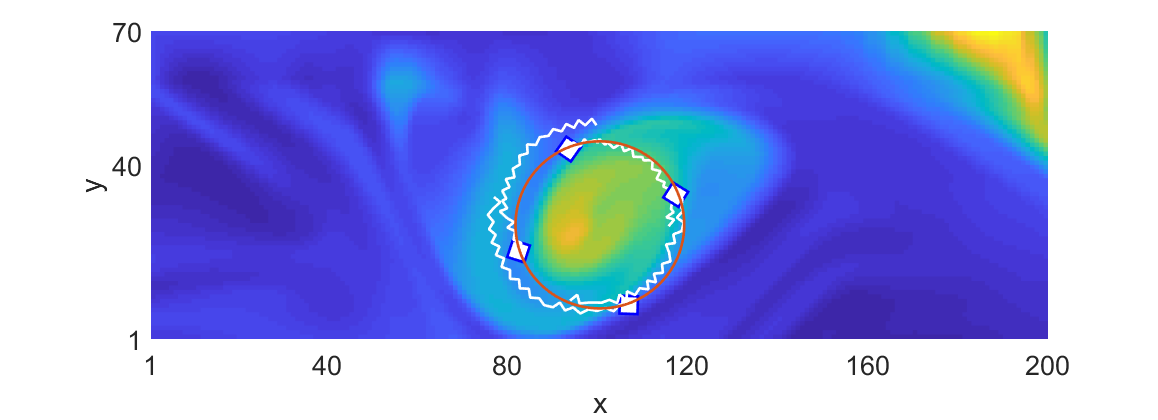}
  \end{subfigure}
  \begin{subfigure}[t]{.49\textwidth}
    \centering
    \includegraphics[width=\linewidth,trim={0.5cm 0 0.8cm 0.6cm},clip]{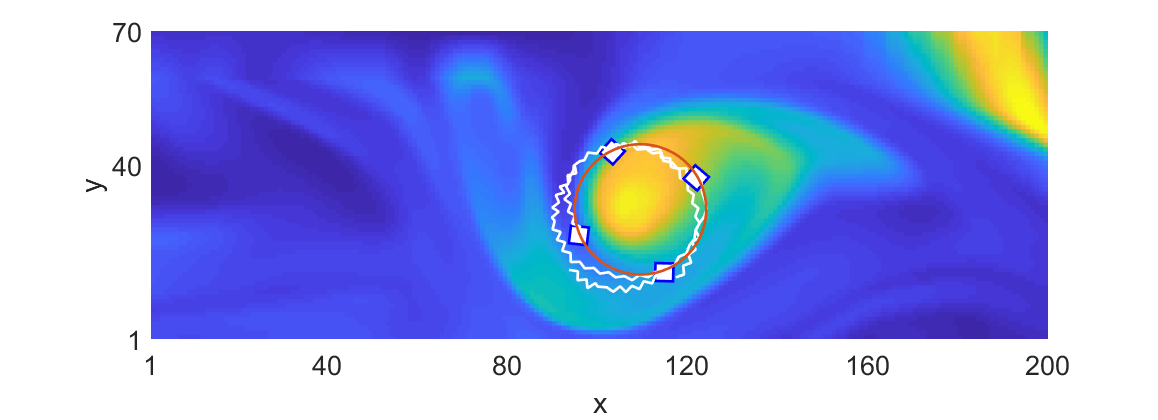}
  \end{subfigure}
    \begin{subfigure}[t]{.49\textwidth}
    \centering
    \includegraphics[width=\linewidth,trim={0.5cm 0 0.8cm 0.6cm},clip]{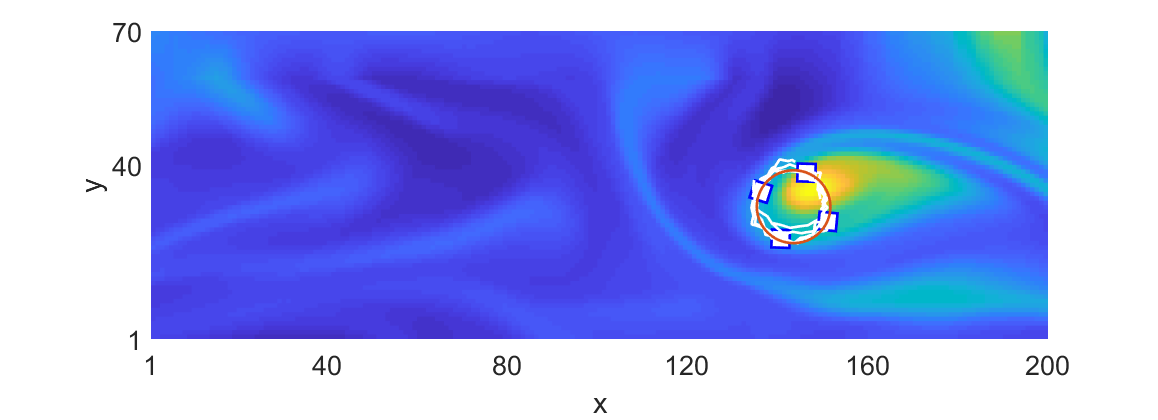}
  \end{subfigure}
    \begin{subfigure}[t]{.49\textwidth}
    \centering
    \includegraphics[width=\linewidth,trim={0.5cm 0 0.8cm 0.6cm},clip]{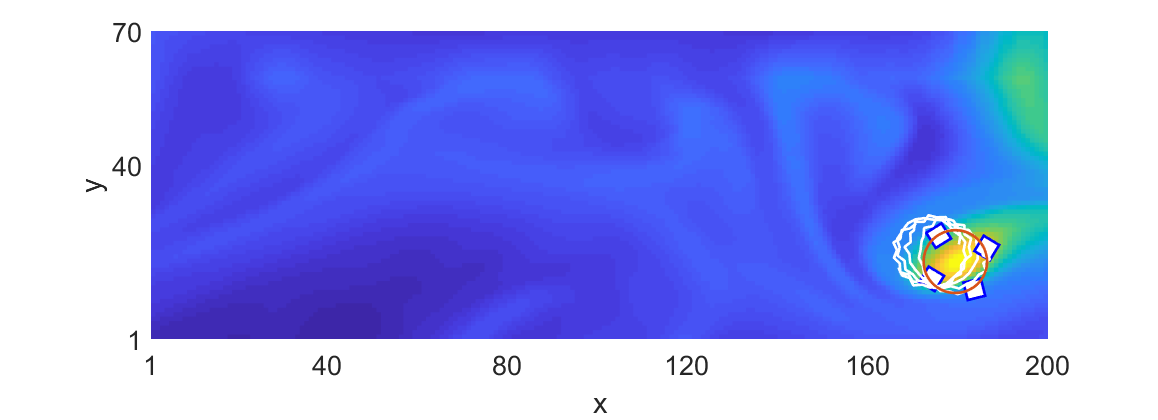}
  \end{subfigure}
  \caption{Time-lapse of four agents circumnavigating a moving target (red) with representation of their paths (white). Each plot is approximately half a day after the previous.}
  \label{paths}
\end{figure}
    
In Fig.\ref{paths} we can see the robot system circumnavigating the algal bloom target in a time-lapse. This specific algal bloom target is quite a challenge as it shape shifts quite abruptly.
Note that the agents were deployed in positions in the boundary so their initial error $D^b_i(0)$ is zero. Note also how, in some instances of the mission, the target moves fast to such extent that the robots present a delay. This effect is foreseen and explained in Theorem \ref{Theo1}. 

\begin{figure}
    \centering
    \includegraphics[width=9.5cm,trim={1.7cm 1.5cm 0 1.9cm},clip]{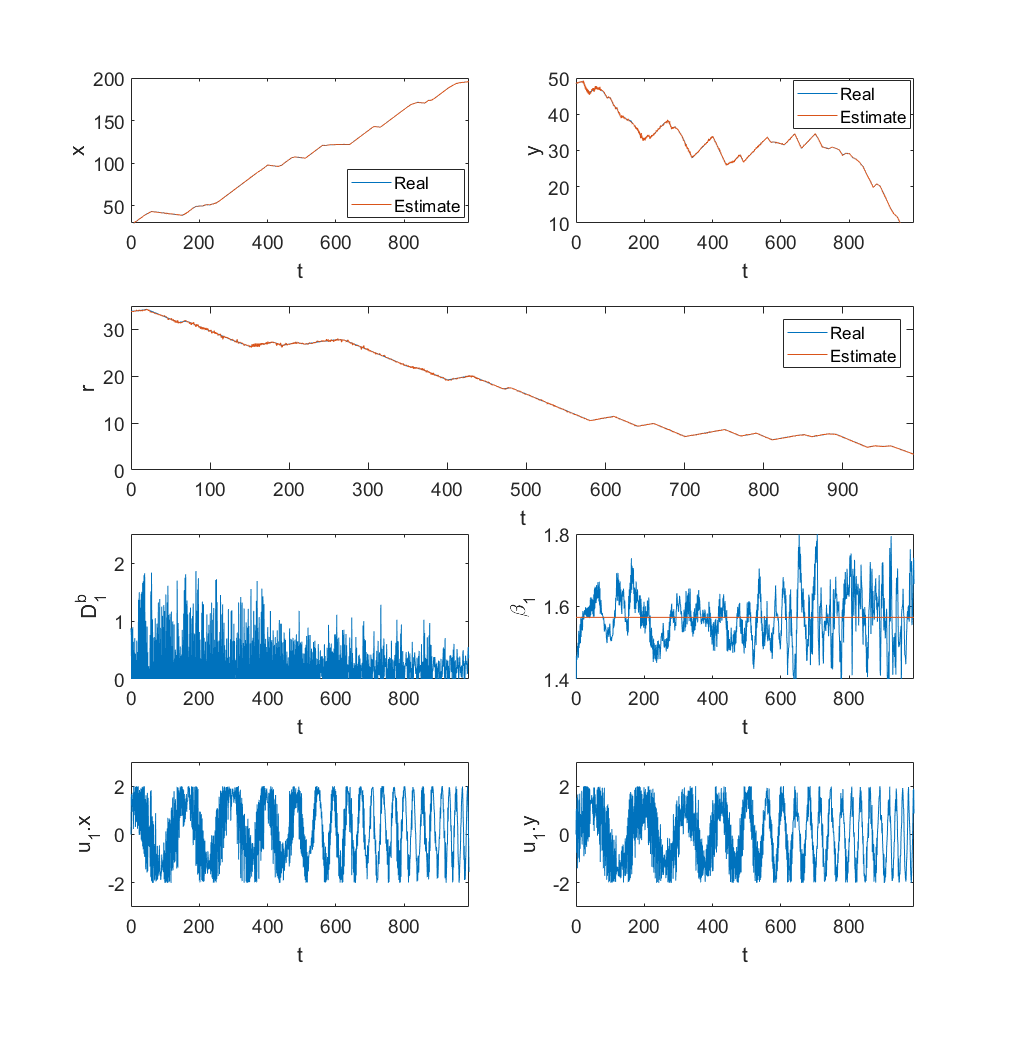}
    \caption{First and second row: real and estimated target's centre $c:x,y$ and radius $r$. Third row: tracking error of agent 1, $D_1^b$ and angle $\beta_1$. Fourth row: control input of agent 1, $u_1:x,y$}
    \label{targets2}
\end{figure}
   
Analysing the simulations, we observe each variable in Fig.\ref{targets2}. Firstly, we can see the comparison between the real position of the target and the estimates our algorithm provided. We can observe that the estimation follows closely the real value with an apparently very small error. Secondly, we analyse the distance of agent 1 to the boundary $D^b_1$ and the angle between agent 1 and 2, $\beta_1$. We can see the error is within the expected boundaries according to Theorem 1. Regarding the distance to the boundary, the error never exceeds 2 units (200 meters) and is most of the time up to 1 unit (100 meters). Note that each x and y coordinate unit corresponds to about 100 meters. Also, each time iteration unit corresponds to 6min. As for the angle between agents the maximum error is 0.2 radians which corresponds to a maximum angle error of 11 degrees. If we look at the plots for the control input of our agents, namely, for agent 1, we can see how the control was applied up to a maximum value. We defined the maximum speed of the agent for each coordinate to be 2 y units per 1 x unit which corresponds to 2km/h in each Cartesian direction (200m / 6min = 2km/h). 


\section{CONCLUSIONS}\label{s: conclusion}

We designed a decentralised algorithm that guarantees circumnavigation of an irregular shape approximated by a circle up to a bounded error. The algorithm relies on one satellite and a number of robots according to the size of the target and to the importance of monitoring its fronts. Then, the proposed control protocol was proven to converge up to a bounded error. 

As future work, we would like to exploit surface vehicles with sensors that measure the point concentration of algal rather than to directly detect the boundary in the local region. Then we would have to explore and circumnavigate collecting data for the estimation. A further objective would be to track any irregular shape which may not be reasonably approximated  by a circle.

\section{ACKNOWLEDGEMENTS}\label{s: conclusion}

We would like to thank Morten Alver and Ingrid Ellingson for providing the SINMOD simulation data.







\bibliographystyle{IEEEtran}
\bibliography{references}

\end{document}